\documentclass[journal]{IEEEtran}

\usepackage{amsmath,amssymb}
\newtheorem{proposition}{Proposition}
\usepackage{graphicx}
\usepackage{booktabs}
\usepackage{multirow}
\usepackage{array}
\usepackage{cite}
\usepackage{xcolor}
\usepackage{url}
\usepackage{graphicx}

\begin{document}
\bstctlcite{paper_bib_control}

\title{Gaussian Material Fields for Volumetric Multi-Energy CT Decomposition}

\author{Jian Lin, Jiancheng Fang, Hongming Shan, \IEEEmembership{Senior Member, IEEE},
Shaoyu Wang, Yang Chen, \IEEEmembership{Senior Member, IEEE},
and Qiegen Liu, \IEEEmembership{Senior Member, IEEE}
\thanks{This work was supported by the National Natural Science Foundation of China (U24A20304), and in part by the National Key Research and Development Program of China (Grant 2023YFF1204300 and Grant 2023YFF1204302). (Jian Lin and Jiancheng Fang are co-first authors.) (Corresponding author: Qiegen Liu.)}
\thanks{J. Lin, J. Fang, S. Wang, and Q. Liu are with the School of Information Engineering, Nanchang University, Nanchang 330031, China (e-mail: liuqiegen@ncu.edu.cn).}
\thanks{H. Shan is with the Institute of Science and Technology for Brain-inspired Intelligence, Fudan University, Shanghai 200433, China (e-mail: hmshan@fudan.edu.cn).}
\thanks{Y. Chen is with the Laboratory of Image Science and Technology, the School of Computer Science and Engineering, and the Key Laboratory of New Generation Artificial Intelligence Technology and Its Interdisciplinary Applications, Ministry of Education, Southeast University, Nanjing 210096, China (e-mail: chenyang.list@seu.edu.cn).}
}


\maketitle

\begin{abstract}
Volumetric material decomposition in multi-energy computed tomography requires a representation that organizes multiple three-dimensional material fields in a common spatial domain while retaining differences in composition and local structure. We observe that spatial primitives can be shared across materials without tying their coefficients, but their local capacity must respond to material-specific reconstruction needs. We introduce Gaussian material fields, which represent multiple material distributions with shared anisotropic 3D Gaussian primitives and independent nonnegative material coefficients. The shared geometry defines a continuous spatial basis, while the coefficients determine each primitive's contribution to the individual material fields. To reconstruct this representation from multi-energy projections, a differentiable spectral forward model combines Gaussian material path integrals with a calibrated basis matrix, enabling joint optimization of spatial geometry and material composition. Material-aware adaptive density control retains material-specific refinement evidence before aggregation and adjusts local representation capacity to accommodate both spatially extensive components and sparse details. Experiments use synthesized multi-energy projections generated from pseudo-reference material maps constructed by conventional methods from publicly available CT data. Across 15 cases, our approach improves average PSNR by 4.03~dB and SSIM by 4.96\% over the strongest baseline, while reducing NRMSE by 33.45\%. Material-wise comparisons and component ablations support improved recovery of localized structures, while runtime and memory measurements show favorable computational scaling. These results establish Gaussian material fields as an explicit, adaptive representation for volumetric multi-material reconstruction.
\end{abstract}


\begin{figure}[t]
\centering
\includegraphics[width=\columnwidth]{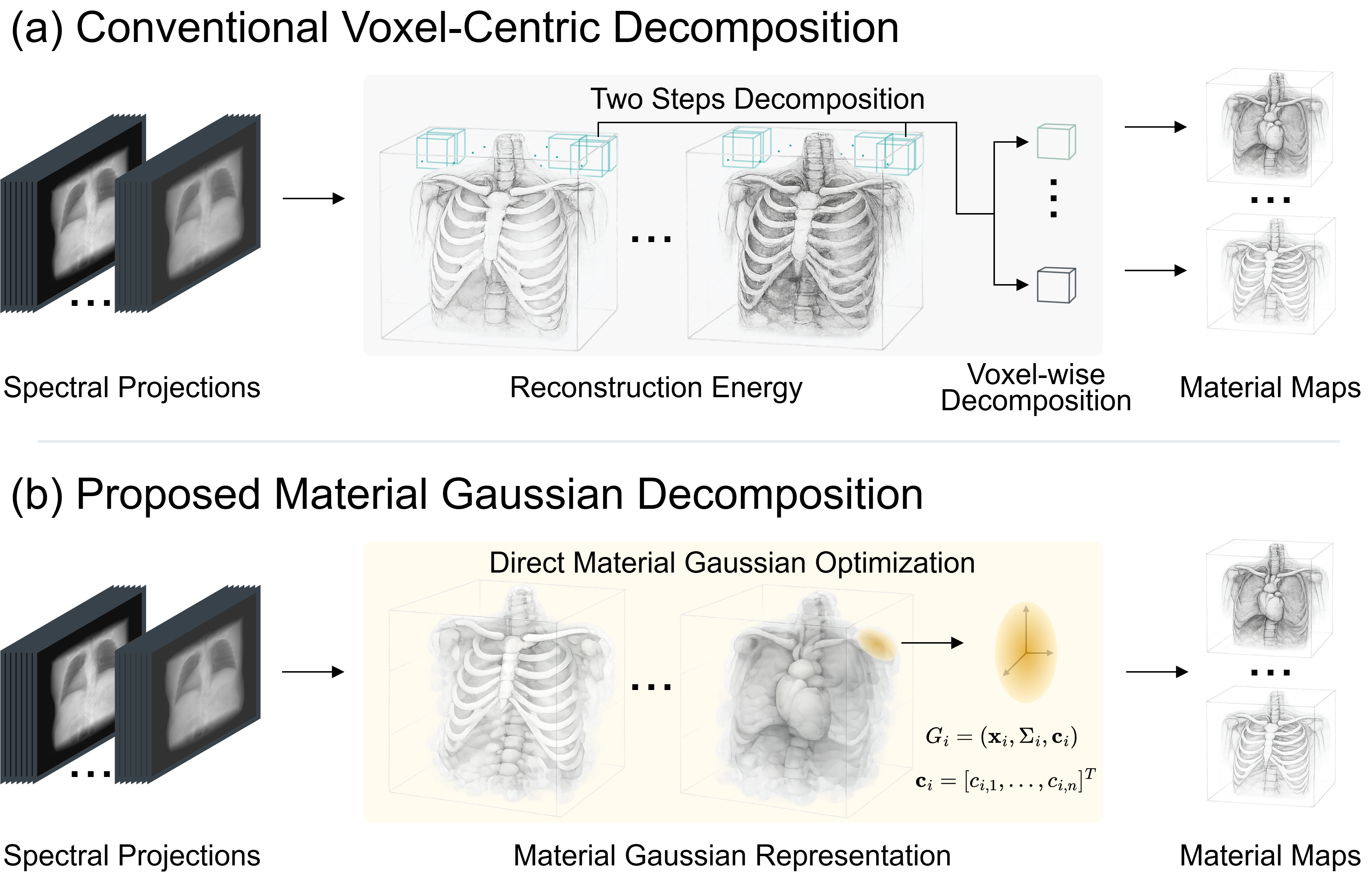}
\caption{Two organizations of volumetric material reconstruction. (a) A two-stage image-domain pipeline reconstructs energy-dependent attenuation volumes before voxel-wise material decomposition. (b) Gaussian material fields represent material composition on shared spatial primitives and optimize geometry and material-specific coefficients from multi-energy projections.}
\label{fig:comparison}
\end{figure}

\section{Introduction}
\IEEEPARstart{M}{ulti-energy} computed tomography (CT), including dual-energy and photon-counting CT, uses energy-dependent X-ray attenuation to recover basis-material distributions beyond conventional attenuation imaging \cite{alvarez1976,mccollough2015,willemink2018,mccollough2020,so2021}. In volumetric material decomposition, the unknown is a collection of three-dimensional material coefficient fields. Recovering these fields requires more than reconstructing several scalar volumes: the representation must organize material composition within a common spatial domain while preserving the distinct distributions and local structures of the individual components.

Existing reconstruction routes organize these unknowns in different ways. Image-domain pipelines first estimate energy-dependent attenuation volumes and then decompose their voxel values. Direct model-based methods instead optimize material volumes through a coupled spectral forward model \cite{long2014,mechlem2018,barber2016,tivnan2020}. Dense voxel grids provide an explicit spatial layout, but assign material variables on a fixed lattice. Neural material fields replace that lattice with a continuous, coordinate-based function \cite{hotta2025,shi2024}, coupling representation and optimization through repeated network queries. Gaussian representations now support both attenuation reconstruction~\cite{cai2024xgs,zha2024r2,li2025_3dgrct,wu2025dgr} and joint spectral basis estimation~\cite{vos2026spectral}. This motivates the question of how a shared spatial basis should be adapted when refinement demands differ across materials.

Our central observation is that spatial organization can be shared while material composition remains independently parameterized. A common set of primitives can define where and over what spatial extent the fields are represented, while separate coefficients determine each primitive's contribution to each material. Shared geometry therefore need not impose equal material values or identical material boundaries. This organization also creates an adaptation requirement: a region that is adequately represented for one material may still lack capacity for another. In particular, a localized component may require refinement even when other materials provide little evidence at the same location. A shared representation should retain this material-specific demand when deciding where to add or remove capacity.

Anisotropic Gaussian primitives provide a concrete way to realize this organization. Their centers, orientations, and scales are explicit spatial parameters, and their overlapping contributions form continuous fields. Material-specific coefficients can weight the same kernels differently, while local primitive edits can change the spatial basis during reconstruction. Differentiable Gaussian X-ray rendering provides the connection to projection measurements \cite{zha2024r2}. Together, these properties make Gaussian primitives suitable not only for compact storage, but also for jointly fitting spatial structure and adapting the basis to heterogeneous material distributions.

We introduce \emph{Gaussian material fields}, which combine shared anisotropic Gaussian geometry with independent nonnegative material coefficients. As illustrated in Fig.~\ref{fig:comparison}, the material fields themselves are the reconstruction variables. A differentiable spectral forward model projects their shared spatial basis, forms material path integrals, and mixes them through a calibrated basis matrix. Projection errors then update geometry and material composition jointly. This coupling also avoids relying on a final image-domain decomposition of separately reconstructed energy fields, where attenuation errors can be amplified. To adapt the shared basis, material-aware adaptive density control retains material-specific evidence before aggregation and reallocates local capacity through cloning, splitting, and pruning. Representation, joint fitting, and capacity adjustment thus address complementary parts of the same multi-material reconstruction problem.

The main contributions are summarized as follows:
\begin{itemize}
    \item We introduce Gaussian material fields, an explicit representation that organizes multiple continuous material distributions through shared spatial primitives and independent material-specific coefficients. A differentiable spectral forward model makes this representation jointly optimizable from multi-energy projections.
    \item We develop material-aware adaptive density control for the shared representation. Material-specific reconstruction evidence guides cloning, splitting, and pruning so that local capacity can respond to both spatially extensive components and sparse material structures.
    \item We evaluate the representation and its reconstruction mechanisms using synthesized multi-energy data derived from public CT data and conventionally constructed pseudo-reference material maps. Overall comparisons, material-wise analysis, component ablations, and configuration studies demonstrate reconstruction quality, detail recovery, and favorable computational characteristics.
\end{itemize}

\section{Related Work}
\subsection{Spectral CT Material Decomposition and Reconstruction}
Spectral CT relates energy-dependent attenuation measurements to basis-material composition~\cite{alvarez1976,mccollough2015,mccollough2020,so2021}. Image-domain decomposition estimates material coefficients from reconstructed energy images. It inherits reconstruction errors despite statistical weighting and cross-channel regularization~\cite{niu2014,xue2017}. Projection-domain decomposition estimates material path integrals before reconstruction, whereas joint reconstruction directly estimates material volumes. Representative methods include E-ART~\cite{zhao2015eart}, IFBP~\cite{li2019ifbp}, SOMA~\cite{pan2023soma}, statistical formulations~\cite{long2014,mechlem2018}, constrained and preconditioned inversion~\cite{barber2016,tivnan2020}, and total-variation regularization~\cite{rudin1992,sidky2008}. Recent networks address image-domain classification and concentration estimation~\cite{rajagopal2025} or learn dual-energy decomposition through polychromatic projection consistency without reference material images~\cite{hellwege2026}.

Continuous neural representations replace fixed voxel grids with coordinate-dependent fields. NeRF introduced this paradigm~\cite{mildenhall2020}, and NeAT, NAF, and IntraTomo adapted neural or self-supervised representations to CT~\cite{rueckert2022,zha2022naf,zang2021}. NeMCoF and neural base-material fields support per-instance material reconstruction without voxel-wise reference supervision~\cite{hotta2025,shi2024}. JSolver jointly estimates material volume fractions and the X-ray spectrum from single-energy projections~\cite{wu2025jsolver}. Diffusion approaches instead introduce learned priors through spectral posterior sampling and memory-efficient volumetric extensions~\cite{jiang2024,jiang2025}. In contrast, Gaussian material fields retain continuous material queries through explicit shared primitives, avoiding dense trainable grids, coordinate-MLP evaluation, and diffusion sampling.

\subsection{Gaussian-Based CT Reconstruction}
3D Gaussian Splatting demonstrated differentiable rendering with explicit anisotropic primitives~\cite{kerbl2023}. Radiative Gaussian Splatting (X-GS) adapted this representation to X-ray novel-view synthesis~\cite{cai2024xgs}, and R$^2$-Gaussian developed a rectified X-ray rasterizer for tomographic reconstruction~\cite{zha2024r2}. Extensions include Gaussian volumetric reconstruction in 3DGR-CT, discretized reconstruction in DGR, coronary reconstruction, and direction-dependent radiograph synthesis~\cite{li2025_3dgrct,wu2025dgr,fu2024_3dgrcar,gao2024ddgs}. Recent work covers faster rasterization and voxelization in FaCT-GS~\cite{pieta2026fact}, learned initialization and projection-domain refinement in SPARK~\cite{zhou2026spark}, and mask-guided initialization with loss scaling in LINGO~\cite{xing2026lingo}.

SpectralCTGaussians fits photoelectric, Compton, and optional K-edge coefficients on shared Gaussians using a polychromatic model, then derives material labels by clustering~\cite{vos2026spectral}. Our method uses a calibrated log-domain model to reconstruct material coefficient fields. Shared geometry organizes multiple continuous material fields, while independent coefficients preserve their distinct compositions. Material-aware density control retains material-branch refinement evidence before aggregation to adapt the shared spatial basis.

\section{Method}
\label{sec:method}

\begin{figure}[t]
\centering
\includegraphics[width=\columnwidth]{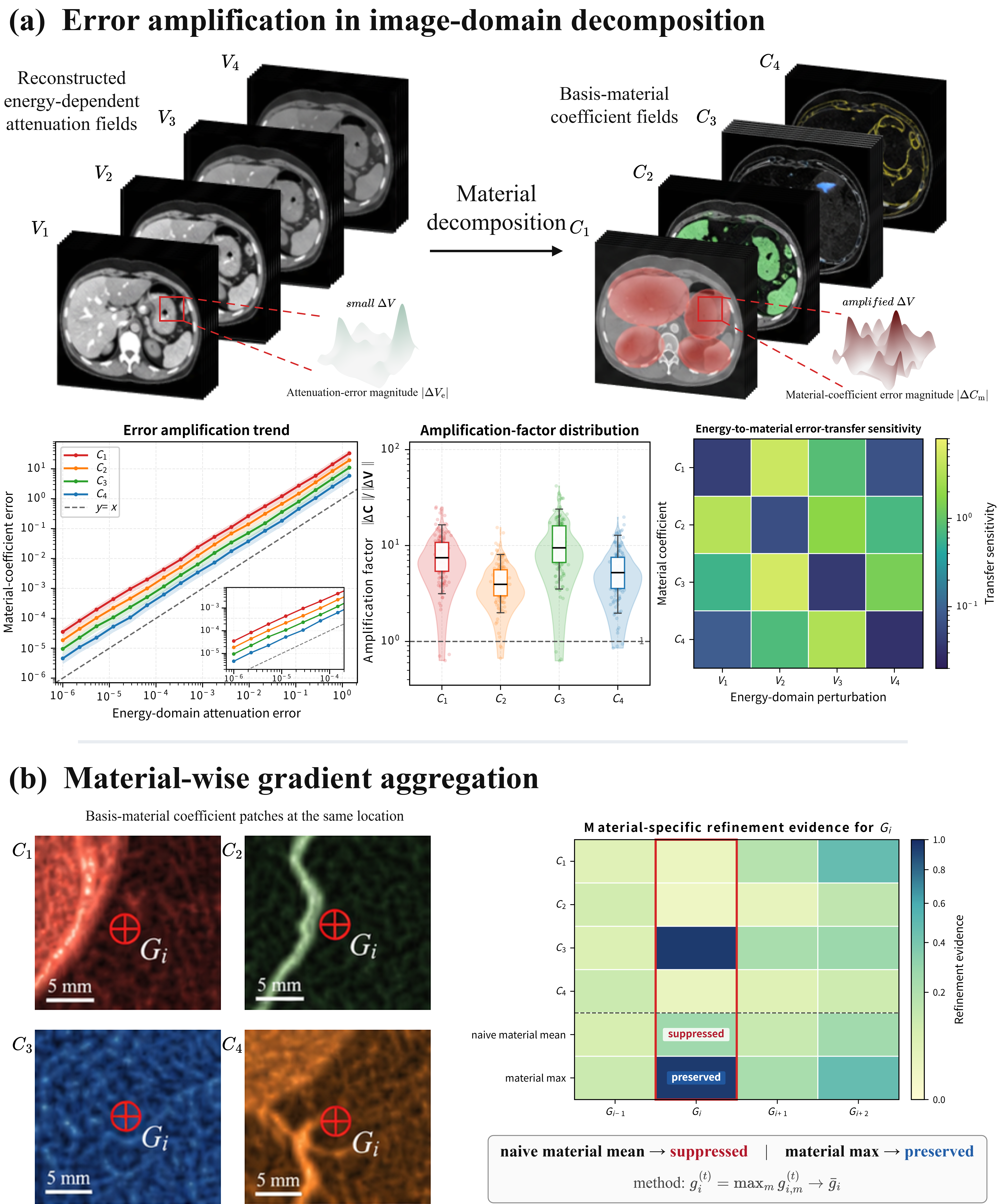}
\caption{(a) Reconstructed energy-dependent attenuation fields $V_e$, corresponding basis-material coefficient fields $C_m$, local error-magnitude visualizations, and quantitative summaries of error transfer across energies and materials. (b) Co-located basis-material coefficient patches and material-specific refinement evidence for a representative Gaussian $G_i$, including naive material averaging and material-wise maximum aggregation.}
\label{fig:motivation}
\end{figure}

\subsection{Motivation}
\label{sec:motivation}

Gaussian material fields separate shared spatial organization from material-specific composition. The method first defines this representation, then fits its geometry and coefficients through a spectral forward model, and finally adjusts its local capacity using material-specific evidence. Fig.~\ref{fig:framework} shows the representation and joint fitting loop. Fig.~\ref{fig:motivation} illustrates the error-transfer and local-refinement considerations addressed by the two reconstruction mechanisms.

\begin{figure*}[t]
\centering
\includegraphics[width=\textwidth]{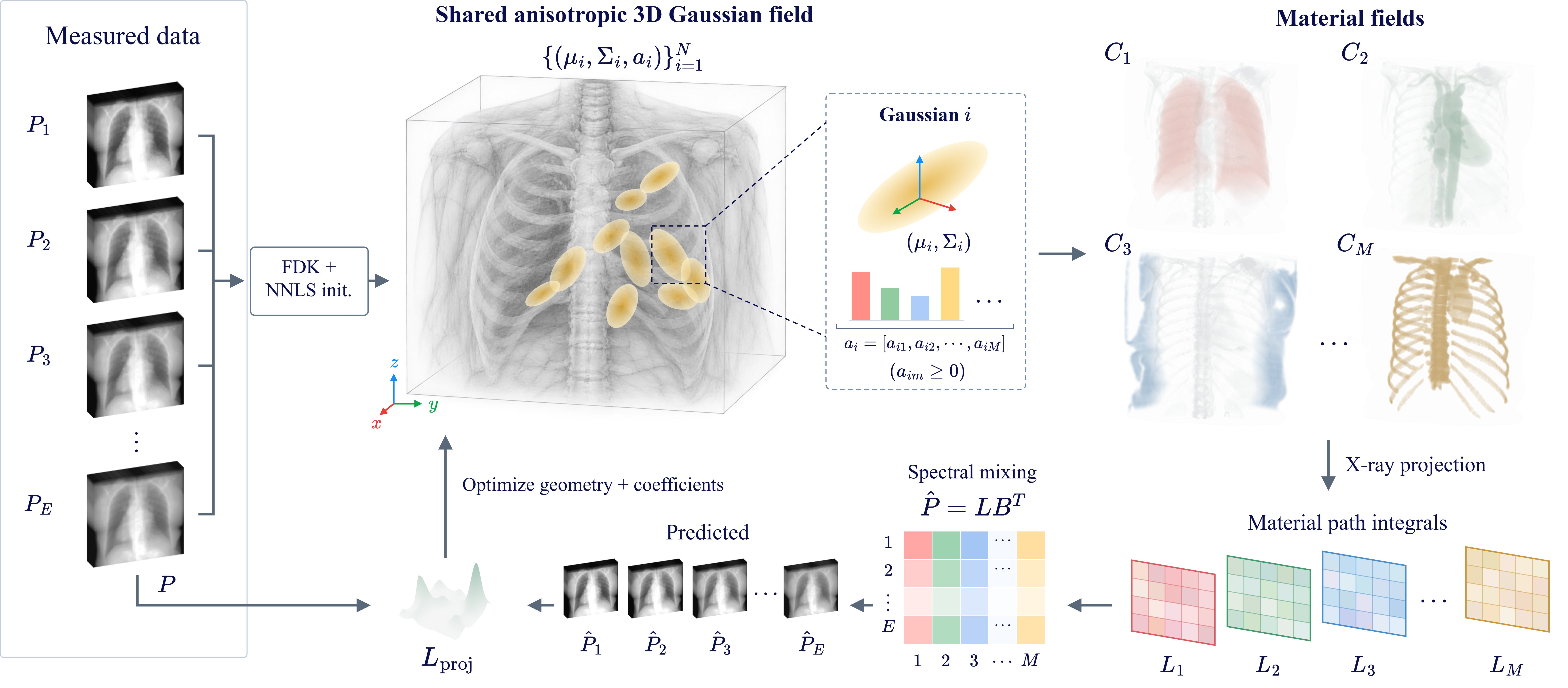}
\caption{Gaussian material fields and their projection-domain optimization. Shared anisotropic primitives carry material-specific coefficients. X-ray projection forms material path integrals, and calibrated spectral mixing produces predicted energy projections. The projection loss updates geometry and coefficients jointly, while the optimized field provides the material outputs.}
\label{fig:framework}
\end{figure*}

\subsection{Shared-Geometry Gaussian Material Field}
\label{sec:method_field}
\label{sec:method_problem}

The representation must place all material fields in a common spatial basis without tying their values. We first specify the measurement model, then define the shared Gaussian kernels and the material-specific coefficients that together form the reconstruction variables.

We reconstruct the nonnegative material field $\mathbf c(\mathbf x)=[C_1(\mathbf x),\ldots,C_M(\mathbf x)]^{\mathsf T}$ directly from $E$ co-registered energy channels, with material indices $\mathcal M=\{1,\ldots,M\}$. Let $\mathbf B\in\mathbb R_{\geq0}^{E\times M}$ be the fixed calibrated basis matrix and $\mathbf b_e^{\mathsf T}$ its $e$th row. For ray $\mathbf x_r(t)=\mathbf o_r+t\mathbf d_r$ with $\|\mathbf d_r\|_2=1$, the effective linear basis model gives
\begin{equation}
\begin{aligned}
 \mathbf v(\mathbf x)&=\mathbf B\mathbf c(\mathbf x),\\
 p_e(r)&=\mathbf b_e^{\mathsf T}\!\int_{t_r^-}^{t_r^+}\mathbf c(\mathbf x_r(t))\,\mathrm dt+\varepsilon_e(r).
\end{aligned}
\label{eq:gmf_measurement_model}
\end{equation}
Here $\mathbf v=[V_1,\ldots,V_E]^{\mathsf T}$ and $\varepsilon_e$ includes noise and model mismatch. This calibrated log-domain model~\cite{alvarez1976,long2014} is not an exact polychromatic photon-count model. Energy-dependent volumes are derived quantities rather than independent reconstruction variables.

All materials share $N$ anisotropic Gaussian kernels, but each kernel carries its own material vector $\mathbf a_i=[a_{i,1},\ldots,a_{i,M}]^{\mathsf T}$. With center $\boldsymbol\mu_i$, positive scales $\boldsymbol\sigma_i$, and rotation $\mathbf R_i$ obtained from a normalized nonzero quaternion $\mathbf q_i$, define
\begin{equation}
\begin{aligned}
 K_i(\mathbf x)&=\exp\!\left[-\tfrac12(\mathbf x-\boldsymbol\mu_i)^{\mathsf T}
 \boldsymbol\Sigma_i^{-1}(\mathbf x-\boldsymbol\mu_i)\right],\\
 \boldsymbol\Sigma_i&=\mathbf R_i\operatorname{diag}(\boldsymbol\sigma_i^2)\mathbf R_i^{\mathsf T}\succ0,\\
 a_{i,m}&=\operatorname{softplus}(\beta_{i,m}).
\end{aligned}
\label{eq:gmf_parameterization}
\end{equation}
We collect the geometric parameters in $\Theta=\{\boldsymbol\mu_i,\mathbf q_i,\boldsymbol\eta_i\}_{i=1}^N$, where $\boldsymbol\eta_i$ parameterizes the Gaussian scales. The material coefficients form the matrix $\mathbf A=[\mathbf a_1,\ldots,\mathbf a_N]^{\mathsf T}=\operatorname{softplus}(\mathbf Z)\in\mathbb R_{>0}^{N\times M}$, where $\mathbf Z=[\beta_{i,m}]$. The material and attenuation fields are then given by
\begin{equation}
\begin{aligned}
 \mathbf c_{\Theta,\mathbf A}(\mathbf x)
 &=\sum_{i=1}^N K_i(\mathbf x)\mathbf a_i
 =\mathbf A^{\mathsf T}\mathbf k(\mathbf x;\Theta),\\
 \mathbf v_{\Theta,\mathbf A}(\mathbf x)&=\mathbf B\mathbf A^{\mathsf T}\mathbf k(\mathbf x;\Theta),
\end{aligned}
\label{eq:gmf_vector_field}
\end{equation}
where $\mathbf k=[K_1,\ldots,K_N]^{\mathsf T}$. Shared geometry establishes spatial correspondence without imposing equal material values or a sum-to-one composition. The kernels are unnormalized: $a_{i,m}$ is a peak amplitude, not integrated material content. Softplus allows arbitrarily small, but not exactly zero, amplitudes at finite logits. Parameter maps and their numerical boundary conventions are discussed in the Appendix.

In Eq.~\eqref{eq:gmf_vector_field}, the kernels determine the shared spatial organization, while the coefficient vectors determine the individual material distributions. Each material field is a differently weighted combination of the same kernels, rather than a copy of a common scalar field. Moving or reshaping a primitive changes the basis available to all materials, whereas changing a material-specific coefficient changes only that primitive's contribution to the corresponding field. This separation supplies the geometric and compositional variables for joint projection-domain fitting.

\subsection{Projection-Domain Joint Reconstruction}
\label{sec:method_optimization}
\label{sec:method_projection}

To fit the material representation, the forward model must connect its shared spatial basis and material-specific coefficients to the energy-resolved measurements. Fig.~\ref{fig:motivation}(a) illustrates why this coupling is useful: errors in reconstructed attenuation fields can be amplified during subsequent image-domain decomposition, with different transfer sensitivities across materials. We instead evaluate projection consistency from the material fields themselves, using the Gaussian ray response and its rasterized implementation below.

Substitution of Eq.~\eqref{eq:gmf_vector_field} into Eq.~\eqref{eq:gmf_measurement_model} reduces each material path integral to a weighted sum of Gaussian ray responses. For $\mathbf Q_i=\boldsymbol\Sigma_i^{-1}$, $\boldsymbol\delta_{i,r}=\mathbf o_r-\boldsymbol\mu_i$, and $\alpha_{i,r}=\mathbf d_r^{\mathsf T}\mathbf Q_i\mathbf d_r$, completing the square yields
\begin{equation}
\begin{aligned}
 \ell_{i,r}^{\infty}
 &=\sqrt{\frac{2\pi}{\alpha_{i,r}}}\exp(-\Delta_{i,r}/2),\\
 \Delta_{i,r}
 &=\boldsymbol\delta_{i,r}^{\mathsf T}\!\left(
 \mathbf Q_i-\frac{\mathbf Q_i\mathbf d_r\mathbf d_r^{\mathsf T}\mathbf Q_i}{\alpha_{i,r}}
 \right)\boldsymbol\delta_{i,r}.
\end{aligned}
\label{eq:gmf_line_response}
\end{equation}
This is the exact infinite-line response. The finite-segment response $\ell_{i,r}^{\mathrm{seg}}=\int_{t_r^-}^{t_r^+}K_i(\mathbf x_r(t))\,\mathrm dt$ and an explicit omitted-tail bound are established in Proposition~\ref{prop:gmf_ray}. Replacing it by Eq.~\eqref{eq:gmf_line_response} requires negligible tails beyond the acquisition segment.

To exploit rasterization, we use the local affine ray-space approximation of R$^2$-Gaussian~\cite{zha2024r2}. Let $\mathbf W_v\in\mathbb R^{3\times3}$ be the rotational part of the viewing transform and $\mathbf J_{i,v}\in\mathbb R^{3\times3}$ the ray-space Jacobian at primitive $i$. With projected center $\widehat{\boldsymbol\mu}_{i,v}$, detector covariance $\widehat{\boldsymbol\Sigma}_{i,v}$, and $\boldsymbol\rho_{i,r}=\mathbf u_r-\widehat{\boldsymbol\mu}_{i,v}$ for ray $r=(v,\mathbf u_r)$,
\begin{equation}
\begin{aligned}
 \widetilde{\boldsymbol\Sigma}_{i,v}
 &=\mathbf J_{i,v}\mathbf W_v\boldsymbol\Sigma_i\mathbf W_v^{\mathsf T}\mathbf J_{i,v}^{\mathsf T},\\
 \Phi_{r,i}&=\zeta_{i,v}\exp\!\left[-\tfrac12\boldsymbol\rho_{i,r}^{\mathsf T}
 \widehat{\boldsymbol\Sigma}_{i,v}^{-1}\boldsymbol\rho_{i,r}\right],\\
 \zeta_{i,v}&=\left(\frac{2\pi|\widetilde{\boldsymbol\Sigma}_{i,v}|}
 {|\widehat{\boldsymbol\Sigma}_{i,v}|}\right)^{1/2}.
\end{aligned}
\label{eq:gmf_projected_gaussian}
\end{equation}
The determinant ratio follows from Gaussian marginalization (Proposition~\ref{prop:gmf_rectifier}), rather than an empirical amplitude correction. The third ray-space coordinate measures physical length. The matrix $\widehat{\boldsymbol\Sigma}_{i,v}$ is the leading $2\times2$ covariance block. Indexing all detector samples by $r=1,\ldots,R$, the shared additive projector gives
\begin{equation}
\begin{aligned}
 \widehat{\mathbf L}&=\boldsymbol\Phi(\Theta)\mathbf A,\\
 \widehat{\mathbf P}&=\boldsymbol\Phi(\Theta)\mathbf A\mathbf B^{\mathsf T},
 \qquad \boldsymbol\Phi\in\mathbb R^{R\times N}.
\end{aligned}
\label{eq:gmf_master_forward}
\end{equation}
Equation~\eqref{eq:gmf_master_forward} is exact for the defined rasterized model, not for the original finite-ray integral. Spectral mixing remains linear, whereas spatial approximation includes local linearization, omitted tails, and numerical footprint truncation. The factorization assumes common geometric footprints and additive material rendering. The matrix denotes an operator applied view by view, not a dense stored array.

\begin{figure*}[t]
\centering
\includegraphics[width=\textwidth]{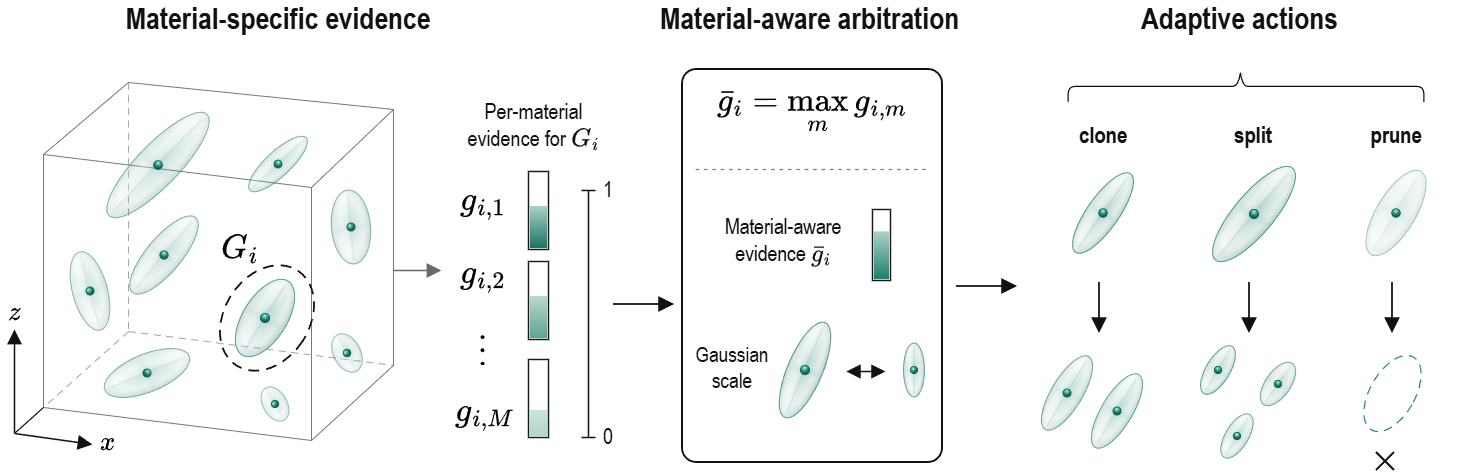}
\caption{Material-aware adaptive density control. Material-specific projection sensitivities are retained before aggregation and combined with Gaussian scale to guide cloning and splitting. Pruning tests material amplitudes across all channels.}
\label{fig:material_density_control}
\end{figure*}

Equation~\eqref{eq:gmf_master_forward} assigns a distinct role to each part of the representation. The shared projector describes how spatial kernels intersect the acquisition rays. The coefficient matrix turns those responses into material path integrals. The calibrated basis matrix mixes the material contributions into energy channels. Geometry and composition therefore affect the same predicted measurements, allowing a common projection objective to update both rather than fitting separate energy-dependent geometries.

For fixed energy weights $\mathbf W_E=\operatorname{diag}(w_1,\ldots,w_E)$, $w_e>0$, the inverse problem at a fixed primitive count is
\begin{equation}
\begin{aligned}
 &\underset{\Theta,\mathbf Z}{\operatorname{minimize}}\quad
 \mathcal L_{\mathrm{proj}}(\Theta,\mathbf Z)
 =\frac1R\|\mathbf E_p\mathbf W_E\|_{1,1},\\
 &\mathbf E_p=\boldsymbol\Phi(\Theta)\operatorname{softplus}(\mathbf Z)
 \mathbf B^{\mathsf T}-\mathbf P,
\end{aligned}
\label{eq:gmf_projection_loss}
\end{equation}
where $\|\mathbf X\|_{1,1}$ sums absolute entries. No volumetric ground truth or auxiliary volume/support loss is used. Define the projection sensitivity $S_{r,e}=(w_e/R)\operatorname{sgn}(E_{p,r,e})$, with $\operatorname{sgn}(0)=0$, and its material-domain counterpart $\mathbf D=\mathbf S\mathbf B$. On a fixed active set and away from zero residuals, the first variation is
\begin{equation}
\begin{aligned}
 \mathrm d\mathcal L_{\mathrm{proj}}
 &=\langle\mathbf D,(\mathrm d\boldsymbol\Phi)\mathbf A
      +\boldsymbol\Phi\,\mathrm d\mathbf A\rangle_F\\
 &=\langle\mathbf D\mathbf A^{\mathsf T},\mathrm d\boldsymbol\Phi\rangle_F
   +\langle\boldsymbol\Phi^{\mathsf T}\mathbf D,\mathrm d\mathbf A\rangle_F.
\end{aligned}
\label{eq:gmf_first_variation}
\end{equation}
Thus, for $\boldsymbol\theta_i\in\{\boldsymbol\mu_i,\mathbf q_i,\boldsymbol\eta_i\}$,
\begin{equation}
\begin{aligned}
 \nabla_{\mathbf A}\mathcal L_{\mathrm{proj}}&=\boldsymbol\Phi^{\mathsf T}\mathbf D,\\
 \nabla_{\mathbf Z}\mathcal L_{\mathrm{proj}}
 &=(\boldsymbol\Phi^{\mathsf T}\mathbf D)\odot\operatorname{sigmoid}(\mathbf Z),\\
 \nabla_{\boldsymbol\theta_i}\mathcal L_{\mathrm{proj}}
 &=\sum_{r=1}^R(\mathbf D\mathbf A^{\mathsf T})_{r,i}
 \nabla_{\boldsymbol\theta_i}\Phi_{r,i}.
\end{aligned}
\label{eq:gmf_gradients}
\end{equation}
At zero residuals these expressions use the selected absolute-value subgradient. Since $D_{r,m}=\sum_e S_{r,e}B_{e,m}$, material coefficients receive spectrally specific updates while geometry receives coupled evidence from every material. Full geometric differentials, including the position dependence of $\mathbf J_{i,v}$, are given in Proposition~\ref{prop:gmf_differentials}. Fixed-geometry separation and conditioning are analyzed in Proposition~\ref{prop:gmf_conditioning}. This analysis neither asserts global uniqueness nor replaces the actual $\ell_1$ objective by least squares.

We initialize from energy-wise FDK followed by voxel-wise NNLS~\cite{feldkamp1984,lawson1995}, then optimize geometry and composition jointly with Adam using separate parameter groups and exponential learning-rate schedules. Initialization details are deferred to the Appendix. Parameter updates are interleaved with the density control below. At the final iterate, material values are queried directly as $\widehat{\mathbf c}(\mathbf x)=\mathbf A_T^{\mathsf T}\mathbf k(\mathbf x;\Theta_T)$. No final independent decomposition is required.

\subsection{Material-Aware Adaptive Density Control}
\label{sec:method_densification}

Joint parameter updates fit the current spatial basis, but cannot by themselves change how many primitives represent a local structure. Capacity adjustment must also respect material differences: a region may be adequately represented in one field yet unresolved in another. Fig.~\ref{fig:motivation}(b) illustrates how averaging across materials can suppress a localized refinement response. We retain the material branches when constructing refinement evidence, then use their aggregation and primitive scale to guide the shared representation's adaptation, as illustrated in Fig.~\ref{fig:material_density_control}.

Equation~\eqref{eq:gmf_gradients} updates parameters within the current representation, whereas density control adjusts its local capacity. To retain material identity, let $\mathbf u_{i,m}^{(v)}$ denote the projected-center variable in material branch $m$, evaluated at the common center $\widehat{\boldsymbol\mu}_{i,v}$. Holding projected covariance and amplitude correction fixed for this partial derivative, Eqs.~\eqref{eq:gmf_projected_gaussian} and \eqref{eq:gmf_first_variation} give
\begin{equation}
\begin{aligned}
 \mathbf h_{i,m}^{(v)}
 &:=\nabla_{\mathbf u_{i,m}^{(v)}}\mathcal L_{\mathrm{proj}}\\
 &=a_{i,m}\sum_{r\in\mathcal R_v}D_{r,m}\Phi_{r,i}
 \widehat{\boldsymbol\Sigma}_{i,v}^{-1}\boldsymbol\rho_{i,r},
\end{aligned}
\label{eq:gmf_material_evidence}
\end{equation}
where $\mathcal R_v$ indexes the view's detector samples. This projected-center contribution is not the full three-dimensional position gradient. Summing the vectors across materials before forming a refinement score can cancel responses. Averaging their norms can dilute a response confined to one material. Writing $\mathbf h_{i,m}^{(t)}$ for the branch derivative at iteration $t$, we use
\begin{equation}
\begin{aligned}
 g_i^{(t)}&=\max_{m\in\mathcal M}\|\mathbf h_{i,m}^{(t)}\|_2,\\
 \bar g_i&=\frac{\sum_{t\in\mathcal W}\nu_i^{(t)}g_i^{(t)}}
 {\sum_{t\in\mathcal W}\nu_i^{(t)}+\epsilon},
\end{aligned}
\label{eq:gmf_refinement_score}
\end{equation}
with zero evidence for invisible branches, visibility indicator $\nu_i^{(t)}$, accumulation window $\mathcal W$, and $\epsilon>0$. The maximum is taken before temporal accumulation. Proposition~\ref{prop:gmf_evidence} states precisely what this aggregation preserves.

For $s_i=\|\boldsymbol\sigma_i\|_\infty$ and thresholds $\tau_g,\tau_s$, the pre-pruning refinement rule is
\begin{equation}
 \mathcal D_i=\begin{cases}
 \mathrm{clone},&\bar g_i\geq\tau_g,\ s_i\leq\tau_s,\\
 \mathrm{split},&\bar g_i\geq\tau_g,\ s_i>\tau_s,\\
 \mathrm{keep},&\bar g_i<\tau_g.
 \end{cases}
\label{eq:gmf_refine_rule}
\end{equation}
These edits remain connected to the same forward model. Replacing primitive $i$ by children $\mathcal C_i$ changes the predicted data by
\begin{equation}
 \delta\widehat{\mathbf P}_i
 =\sum_{j\in\mathcal C_i}\boldsymbol\phi_{i,j}(\mathbf B\mathbf a_{i,j})^{\mathsf T}
 -\boldsymbol\phi_i(\mathbf B\mathbf a_i)^{\mathsf T},
\label{eq:gmf_edit_response}
\end{equation}
where $\boldsymbol\phi_i=\boldsymbol\Phi_{:,i}$. Cloning creates two identical kernels with activated coefficients $\mathbf a_i/2$, so Eq.~\eqref{eq:gmf_edit_response} vanishes at insertion under unchanged linear rendering. Splitting samples center offsets using the parent covariance, reduces scales, and divides coefficients across children. It preserves the coefficient sum, but generally changes both the field integral and its projections. Exact rules and their consequences are established in Proposition~\ref{prop:gmf_edits}.

A low-content primitive is pruned only when $\max_m a_{i,m}<\tau_a$. The geometric guard also removes centers outside the valid reconstruction box $\Omega$. Thus, a non-negligible response in any material prevents low-content pruning. Refinement is restricted by the prescribed iteration window and Gaussian budget. Optional size tests apply only when enabled. These residual-guided edits change the approximation space between parameter updates, without constituting an exact minimization over $N$ or guaranteeing a loss decrease. Together, Eqs.~\eqref{eq:gmf_projection_loss} and \eqref{eq:gmf_edit_response} describe fitting and adapting one continuous material representation.

\section{Experiments}
\label{sec:experiments}

\subsection{Datasets and Experimental Settings}
\label{sec:exp_setup}

We report material reconstruction on 15 simulated multi-energy cone-beam cases. Of these, 9 anatomical soft-tissue/calcium phantoms are constructed from AAPM Low Dose CT Grand Challenge CT volumes~\cite{mccollough2017lowdose} and have known synthetic material maps. The remaining 6 cases use material maps derived from lizard-head data~\cite{warr2021lizarddata,warr2021hyperspectral}, powder-phantom data~\cite{warr2022powderdata,warr2021hyperspectral}, mouse-limb data~\cite{warr2022mousedata,warr2022mouse}, MUSIC 3D~\cite{kehl2018music}, walnut data~\cite{zhou2025walnutpcct}, and ore data~\cite{warr2020ore}. We refer to these derived maps as pseudo-reference material maps. The material maps generate the multi-energy projections and serve as the evaluation references. Table~\ref{tab:datasets} summarizes the channel counts and reconstruction grids.

\begin{table}[!t]
\caption{Simulated cases. $E/M$ gives energy/material channels.}
\label{tab:datasets}
\centering
\footnotesize
\setlength{\tabcolsep}{2pt}
\begin{tabular*}{\columnwidth}{@{\extracolsep{\fill}}p{1.72in}ccc@{}}
\toprule
Case & $E/M$ & Detector & Volume \\
\midrule
AAPM & $2/2$ & $512^2$ & $256^3$ \\
Lizard head, mouse forelimb & $3/3$ & $128^2$ & $80^3$ \\
Powder phantom & $4/4$ & $128^2$ & $80^3$ \\
MUSIC 3D & $6/6$ & $256^2$ & $256^3$ \\
Walnut (cubic) & $2/2$ & $256^2$ & $256^3$ \\
Ore (cubic) & $2/2$ & $128^2$ & $80^3$ \\
\bottomrule
\end{tabular*}
\end{table}

For every case, 50 uniformly spaced training views cover a full rotation and 100 nonoverlapping views form the test set. The AAPM training projections include the prescribed Poisson and Gaussian noise, whereas their test projections and the projections in the other 6 cases are clean. The cone-beam simulator uses source-to-detector and source-to-origin distances of 7 and 5, respectively, and a $4\times4$ detector extent in simulation coordinates. Detector sampling and volume resolution are listed in Table~\ref{tab:datasets}.

The comparison includes FDK+NNLS~\cite{feldkamp1984,lawson1995}, ASD-POCS+NNLS~\cite{sidky2008}, NeRF+NNLS~\cite{mildenhall2020}, R2+NNLS~\cite{zha2024r2}, ADJUST~\cite{zeegers2022adjust}, CIL SIRT+NNLS~\cite{jorgensen2021cil}, and the proposed method. The 5 NNLS baselines reconstruct energy volumes before material decomposition. The dataset's material basis matrix and channel order are fixed for each case and used by all methods that require the basis. The compared runs use 6 ASD-POCS iterations, 1000 NeRF epochs, 15,000 R2 iterations per energy channel, 1000 outer and 10 inner ADJUST iterations, 50 CIL SIRT iterations, and 15,000 iterations for the proposed method. FDK is a direct reconstruction. These are method-specific stopping settings rather than a common iteration budget. The source code is available at \url{https://github.com/yqx7150/MGS}.

We evaluate every method on the same reference material channels for each case. For each channel, PSNR uses the range of its reference volume, NRMSE is voxelwise RMSE divided by that range, and SSIM averages nonempty reference slices along the 3 volume axes. We report the arithmetic mean of each metric across a case's material channels. Family-level summaries average 1 score per case configuration.

\subsection{Comparative Results}
\label{sec:exp_comparison}

Table~\ref{tab:main_results} presents all 15 simulated-projection cases at 50 training views. Each cell gives material-averaged full-volume PSNR, SSIM, and NRMSE, and the final column averages the 15 cases equally. The proposed method leads all 3 metrics in each case.

\begin{table*}[!t]
\caption{50-view material-volume PSNR (dB)/SSIM/NRMSE on 15 cases.}
\label{tab:main_results}
\centering
\fontsize{8}{9}\selectfont
\setlength{\tabcolsep}{1pt}
\renewcommand{\arraystretch}{1.06}
\begin{tabular*}{\textwidth}{@{\extracolsep{\fill}}lccccc@{}}
\toprule
Method & L096 & L109 & L143 & L192 & L286 \\ 
\midrule
FDK+NNLS & 23.58/0.3316/0.0705 & 18.33/0.2113/0.1279 & 17.10/0.1789/0.1489 & 19.90/0.2393/0.1065 & 18.01/0.1943/0.1329 \\ 
ADJUST & 31.15/0.8710/0.0370 & 26.68/0.8034/0.0586 & 24.15/0.6455/0.0850 & 27.20/0.7694/0.0528 & 25.20/0.6877/0.0692 \\ 
NeRF+NNLS & 31.10/0.7799/0.0294 & 27.66/0.6679/0.0434 & 28.29/0.6608/0.0409 & 29.28/0.7530/0.0359 & 28.56/0.6764/0.0393 \\ 
CIL SIRT+NNLS & 31.35/0.8079/0.0314 & 28.10/0.6754/0.0428 & 27.31/0.5913/0.0495 & 27.88/0.7211/0.0428 & 27.52/0.6415/0.0458 \\ 
R2+NNLS & \underline{34.02}/0.8610/\underline{0.0216} & \underline{31.32}/0.7987/\underline{0.0289} & 29.95/0.7538/\underline{0.0351} & \underline{31.76}/0.8331/\underline{0.0274} & \underline{31.29}/0.7959/\underline{0.0296} \\ 
ASD-POCS+NNLS & 33.60/\underline{0.9011}/0.0243 & 30.75/\underline{0.8340}/0.0319 & \underline{30.17}/\underline{0.8384}/0.0375 & 30.25/\underline{0.8698}/0.0328 & 30.19/\underline{0.8551}/0.0342 \\ 
Ours & \textbf{37.01}/\textbf{0.9340}/\textbf{0.0169} & \textbf{33.55}/\textbf{0.8939}/\textbf{0.0243} & \textbf{32.82}/\textbf{0.8797}/\textbf{0.0286} & \textbf{33.68}/\textbf{0.9109}/\textbf{0.0232} & \textbf{33.62}/\textbf{0.8923}/\textbf{0.0246} \\ 
\bottomrule
\end{tabular*}
\par\vspace{5pt}
\begin{tabular*}{\textwidth}{@{\extracolsep{\fill}}lccccc@{}}
\toprule
Method & L291 & L310 & L333 & L506 & Lizard \\ 
\midrule
FDK+NNLS & 24.48/0.3551/0.0632 & 16.83/0.1630/0.1522 & 19.19/0.2180/0.1155 & 18.69/0.2163/0.1221 & 27.68/0.6972/0.0449 \\ 
ADJUST & 31.71/0.8971/0.0322 & 24.12/0.6371/0.0835 & 28.16/0.8303/0.0486 & 25.13/0.7236/0.0652 & 25.59/0.7971/0.0620 \\ 
NeRF+NNLS & 31.78/0.8238/0.0270 & 28.25/0.6357/0.0408 & 28.78/0.7207/0.0380 & 27.73/0.7169/0.0428 & 27.77/0.7940/0.0452 \\ 
CIL SIRT+NNLS & 30.90/0.8302/0.0313 & 27.33/0.5912/0.0473 & 27.74/0.7035/0.0438 & 27.10/0.6988/0.0462 & 29.82/0.8683/0.0325 \\ 
R2+NNLS & \underline{34.28}/0.8828/\underline{0.0205} & \underline{30.90}/0.7681/\underline{0.0312} & \underline{31.79}/0.8303/\underline{0.0275} & \underline{30.92}/0.8199/\underline{0.0300} & 29.91/0.8719/0.0357 \\ 
ASD-POCS+NNLS & 33.02/\underline{0.9056}/0.0244 & 30.38/\underline{0.8613}/0.0346 & 30.17/\underline{0.8681}/0.0334 & 29.57/\underline{0.8607}/0.0349 & \underline{31.47}/\underline{0.9321}/\underline{0.0268} \\ 
Ours & \textbf{36.65}/\textbf{0.9431}/\textbf{0.0165} & \textbf{33.65}/\textbf{0.8841}/\textbf{0.0255} & \textbf{33.85}/\textbf{0.9084}/\textbf{0.0232} & \textbf{33.06}/\textbf{0.9053}/\textbf{0.0248} & \textbf{33.92}/\textbf{0.9581}/\textbf{0.0203} \\ 
\bottomrule
\end{tabular*}
\par\vspace{5pt}
\begin{tabular*}{\textwidth}{@{\extracolsep{\fill}}lcccccc@{}}
\toprule
Method & Mouse FL & Powder & MUSIC & Walnut & Ore & Average \\ 
\midrule
FDK+NNLS & 30.05/0.7778/0.0326 & 26.55/0.6505/0.0746 & 42.96/0.8691/0.0078 & 26.68/0.4796/0.0470 & 27.68/0.7866/0.0413 & 23.85/0.4246/0.0859 \\ 
ADJUST & 23.00/0.7398/0.0723 & 28.10/0.8178/0.0572 & 46.67/0.9890/0.0051 & 21.94/0.7501/0.0810 & 24.22/0.7264/0.0657 & 27.54/0.7790/0.0584 \\ 
NeRF+NNLS & 22.33/0.7117/0.0879 & 27.44/0.7122/0.0622 & 44.17/0.9719/0.0071 & 22.00/0.7011/0.0797 & 32.44/0.8814/0.0239 & 29.17/0.7471/0.0429 \\ 
CIL SIRT+NNLS & 31.83/0.9324/0.0266 & 29.50/0.7772/0.0469 & 43.52/0.9832/0.0072 & 27.14/0.8110/0.0444 & 31.81/0.8885/0.0257 & 29.92/0.7681/0.0376 \\ 
R2+NNLS & 24.88/0.7760/0.0661 & 29.19/0.7261/0.0523 & \underline{48.86}/0.9901/\underline{0.0040} & 22.20/0.7209/0.0781 & \underline{33.34}/\underline{0.9071}/\underline{0.0216} & 31.64/0.8224/0.0340 \\ 
ASD-POCS+NNLS & \underline{33.01}/\underline{0.9481}/\underline{0.0230} & \underline{32.98}/\underline{0.8826}/\underline{0.0322} & 45.58/\underline{0.9904}/0.0056 & \underline{29.86}/\underline{0.8756}/\underline{0.0324} & 31.24/0.8751/0.0274 & \underline{32.15}/\underline{0.8865}/\underline{0.0290} \\ 
Ours & \textbf{35.58}/\textbf{0.9711}/\textbf{0.0174} & \textbf{41.39}/\textbf{0.9579}/\textbf{0.0108} & \textbf{51.18}/\textbf{0.9964}/\textbf{0.0029} & \textbf{37.77}/\textbf{0.9712}/\textbf{0.0133} & \textbf{35.01}/\textbf{0.9515}/\textbf{0.0178} & \textbf{36.18}/\textbf{0.9305}/\textbf{0.0193} \\ 
\bottomrule
\end{tabular*}
\end{table*}

On the 9 AAPM phantoms, the proposed method's mean PSNR margin over the strongest competing method in each case is 2.38~dB. The 2 animal cases also favor the proposed method. The largest gaps occur for the powder phantom and walnut, where the proposed method reaches 41.39 and 37.77~dB, versus 32.98 and 29.86~dB for the strongest alternatives. These results show improved recovery of the constructed and pseudo-reference material volumes across the tested configurations.

In Fig.~\ref{fig:qualitative_anatomy}, the proposed method preserves the thin calcification arc and cortical boundary more continuously than competing methods. It also retains the forelimb contour and internal structure with fewer spurious responses.

\begin{figure*}[!t]
\centering
\includegraphics[width=\textwidth]{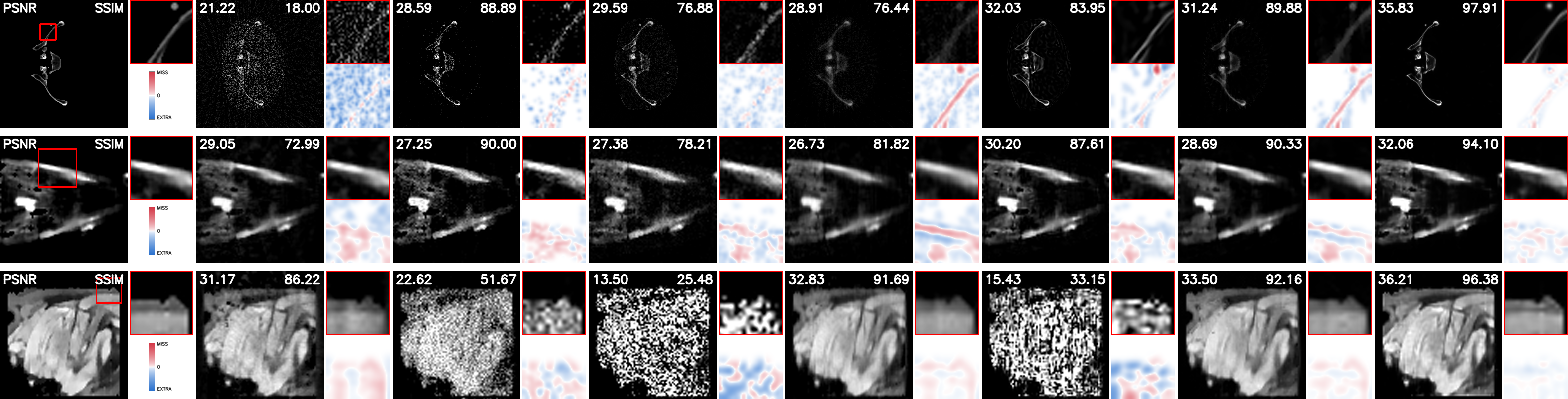}
\caption{Matched 50-view slices. Rows, top to bottom: AAPM L506 calcification, lizard-head bone, mouse-forelimb background. Columns, left to right: GT, FDK, ADJUST, NeRF, CIL, R2, ASD-POCS, Ours. Main/ROI views are at left/upper right. Auxiliary structural differences are at lower right (red: missing, blue: extra). Window $[0,1]$. Full-slice 2D PSNR (dB)/SSIM (\%).}
\label{fig:qualitative_anatomy}
\end{figure*}

\begin{figure*}[!t]
\centering
\includegraphics[width=\textwidth]{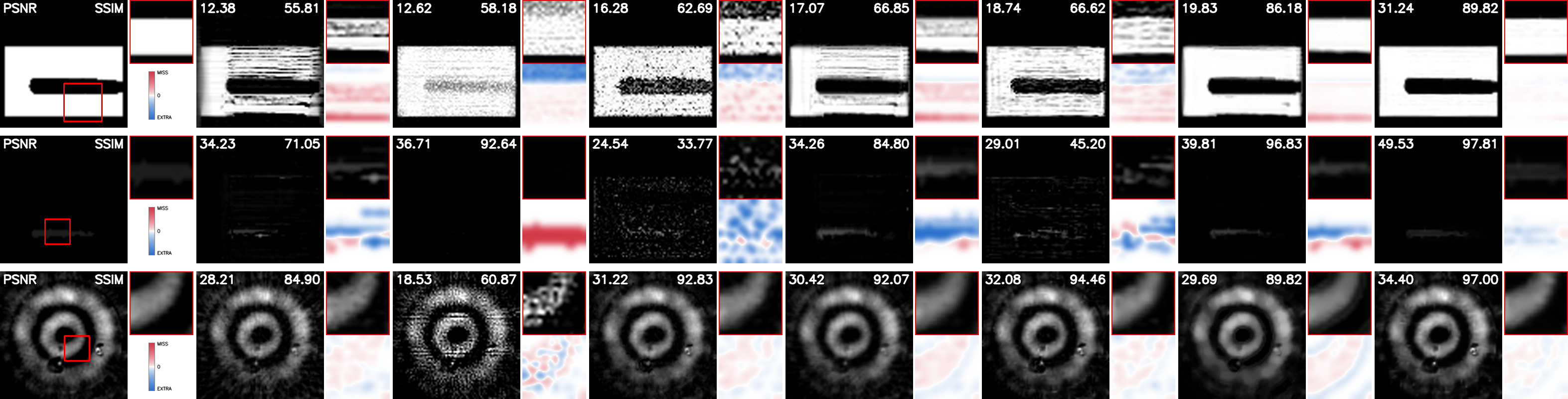}
\caption{50-view material-slice comparisons. Rows show powder-phantom Al, powder-phantom Fe, and ore Pb, respectively. Enlarged regions highlight narrow powder features and the inner ore-Pb ring.}
\label{fig:qualitative_industrial}
\end{figure*}

\begin{figure*}[!t]
\centering
\includegraphics[width=\textwidth]{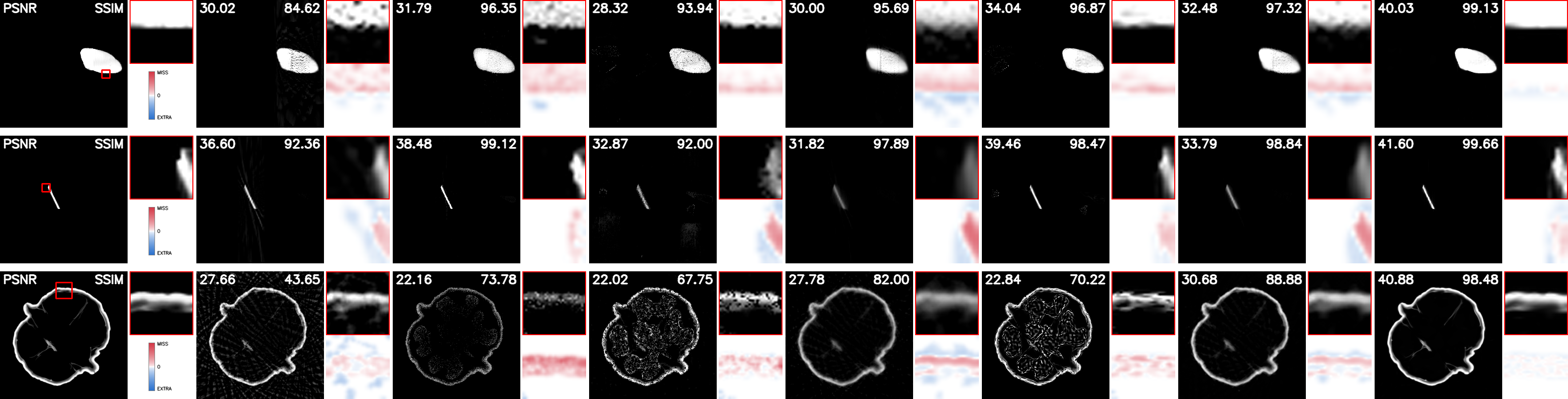}
\caption{50-view material-slice comparisons. Rows show MUSIC 3D water-like, MUSIC 3D iodine-like, and walnut shell, respectively. Enlarged regions highlight compact material boundaries and the thin walnut-shell rim.}
\label{fig:qualitative_diversity}
\end{figure*}

In Fig.~\ref{fig:qualitative_industrial}, the proposed method better preserves narrow powder features against a smoother background and recovers a more continuous inner ore-Pb ring.

In Fig.~\ref{fig:qualitative_diversity}, the proposed method retains compact MUSIC boundaries with less mottling and a more continuous walnut rim, while several baselines introduce signal inside the shell.

\begin{figure}[!t]
\centering
\includegraphics[width=\columnwidth]{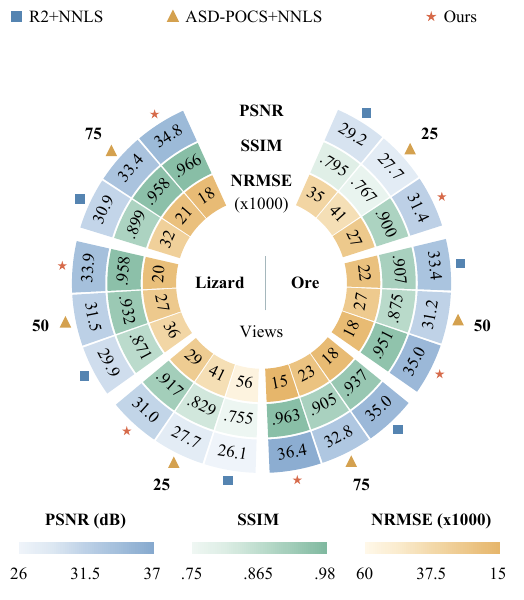}
\caption{Material-volume quality at 25/50/75 views: lizard head (left) and ore (right). Rings show PSNR, SSIM, and NRMSE $\times10^3$ from outside in. Symbols mark methods. Darker shading is better within each metric. Each setting uses 1 seed.}
\label{fig:sparse_views}
\end{figure}

\begin{table}[!ht]
\caption{Material-volume mean $\pm$ SD across 3 seeds. NRMSE $\times10^{-3}$. Bold/underline mark best/second means.}
\label{tab:ablation_results}
\centering
\fontsize{8}{9}\selectfont
\setlength{\tabcolsep}{1.1pt}
\renewcommand{\arraystretch}{1.02}
\begin{tabular*}{\columnwidth}{@{\extracolsep{\fill}}lccc@{}}
\toprule
Setting & PSNR (dB) & SSIM & NRMSE ($10^{-3}$) \\
\midrule
\multicolumn{4}{@{}l}{\emph{Lizard head}} \\
Full & 33.91$\pm$0.03 & 0.95813$\pm$0.00021 & 20.368$\pm$0.071 \\
Fixed set & 33.33$\pm$0.01 & 0.95470$\pm$0.00021 & 21.778$\pm$0.027 \\
Mean agg. & \underline{34.06}$\pm$0.04 & \underline{0.96107}$\pm$0.00015 & \underline{20.045}$\pm$0.083 \\
Material $\ell_1$ & \textbf{34.30}$\pm$0.01 & \textbf{0.96396}$\pm$0.00002 & \textbf{19.538}$\pm$0.027 \\
Energy $\ell_1$ & 29.66$\pm$0.02 & 0.88741$\pm$0.00031 & 34.538$\pm$0.095 \\
\midrule
\multicolumn{4}{@{}l}{\emph{Powder phantom}} \\
Full & 41.55$\pm$0.16 & 0.95820$\pm$0.00024 & 10.728$\pm$0.102 \\
Fixed set & 40.85$\pm$0.03 & \underline{0.96617}$\pm$0.00019 & 11.023$\pm$0.031 \\
Mean agg. & \underline{41.90}$\pm$0.21 & 0.96604$\pm$0.00047 & \underline{10.267}$\pm$0.265 \\
Material $\ell_1$ & \textbf{42.59}$\pm$0.14 & \textbf{0.96730}$\pm$0.00027 & \textbf{9.589}$\pm$0.069 \\
Energy $\ell_1$ & 33.13$\pm$0.08 & 0.80226$\pm$0.00020 & 29.940$\pm$0.287 \\
\midrule
\multicolumn{4}{@{}l}{\emph{MUSIC 3D}} \\
Full & 51.22$\pm$0.06 & 0.99644$\pm$0.00002 & 2.865$\pm$0.022 \\
Fixed set & 50.36$\pm$0.07 & 0.99720$\pm$0.00001 & 3.115$\pm$0.023 \\
Mean agg. & \underline{51.43}$\pm$0.02 & \underline{0.99758}$\pm$0.00001 & \underline{2.784}$\pm$0.012 \\
Material $\ell_1$ & \textbf{51.59}$\pm$0.04 & \textbf{0.99791}$\pm$0.00001 & \textbf{2.731}$\pm$0.009 \\
Energy $\ell_1$ & 51.11$\pm$0.02 & 0.99500$\pm$0.00002 & 2.926$\pm$0.003 \\
\bottomrule
\end{tabular*}
\end{table}

\begin{figure}[!t]
\centering
\includegraphics[width=\columnwidth]{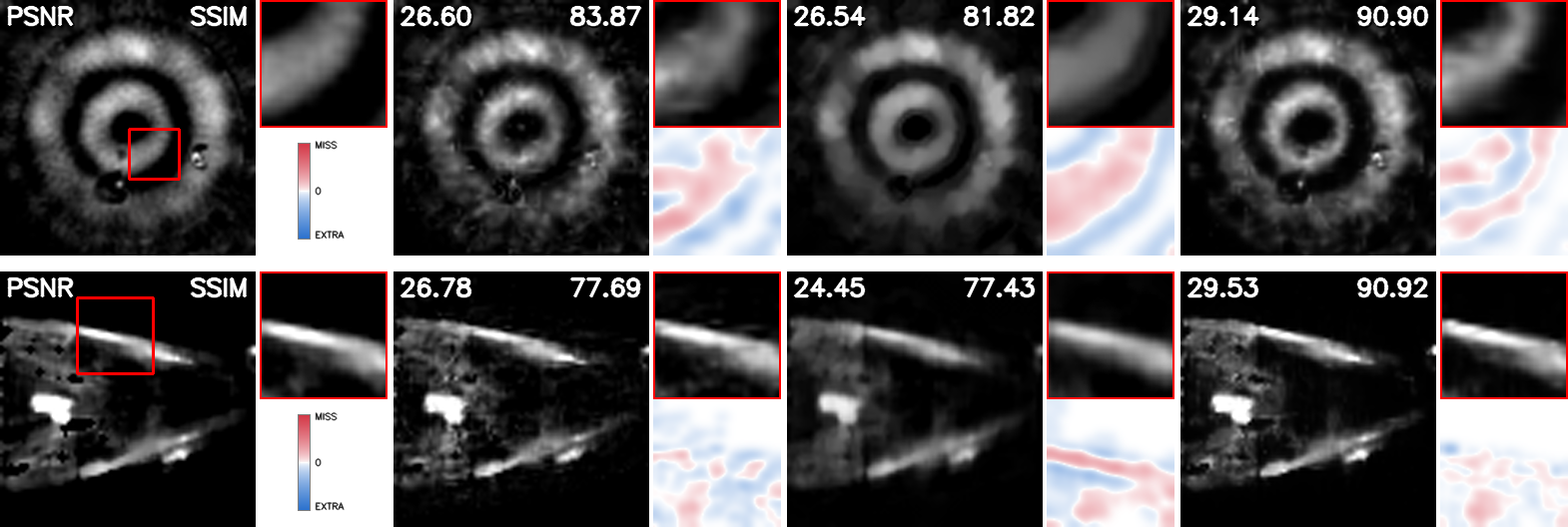}
\caption{Matched 25-view slices: ore Pb and lizard-head bone, top to bottom. Columns: GT, R2, ASD-POCS, Ours.}
\label{fig:sparse_qualitative_v25}
\end{figure}

\begin{figure}[!t]
\centering
\includegraphics[width=\columnwidth]{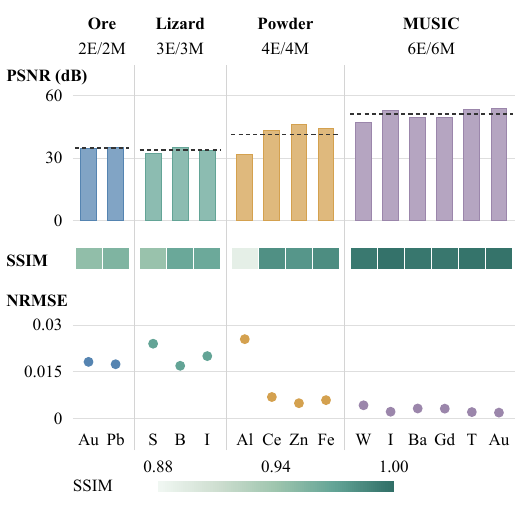}
\caption{Per-material quality in 4 scene-specific energy/material configurations. Bars show PSNR (dB), dashed lines case means, shaded strips SSIM, and dots unscaled NRMSE. S/B denote soft tissue/bone, Ce/Zn denote CeO$_2$/ZnO, and W/T denote water-like/tungsten-like. MUSIC uses a $256^3$ grid, the others $80^3$.}
\label{fig:multienergy_scope}
\end{figure}

We also evaluate full-rotation sparse-view reconstruction on the ore and lizard-head cases using 25, 50, and 75 uniformly spaced training views. Within each case, the material basis, channel order, volume grid, and 100 disjoint test views stay fixed. The proposed method, R2+NNLS, and ASD-POCS+NNLS are evaluated with 1 seed each and run for 15,000 iterations, 15,000 iterations per energy channel, and 6 iterations, respectively. Fig.~\ref{fig:sparse_views} shows the material-averaged full-volume PSNR, SSIM, and NRMSE across view counts. The proposed method leads all 3 metrics at the 6 tested case/view settings.

Fig.~\ref{fig:sparse_qualitative_v25} compares selected 25-view material slices with the same slice, ROI, and grayscale window across methods within each case. The proposed method retains the inner ore-Pb ring and the lizard-head cortical edge more continuously than the 2 shown baselines. Ore Pb uses the same fixed location as its 50-view row in Fig.~\ref{fig:qualitative_industrial}.

Fig.~\ref{fig:multienergy_scope} reports the complete method on 4 distinct simulated-projection cases with 2, 3, 4, and 6 energy/material channels. Each uses 50 training views, 100 disjoint test views, 1 seed, and 15,000 iterations. The material basis is fixed within each case. These observations demonstrate reconstruction for the 4 configurations, but their different scenes and bases, along with the different grid used in 1 case, prevent attributing case differences to channel count alone.

\begin{figure}[!t]
\centering
\includegraphics[width=\columnwidth]{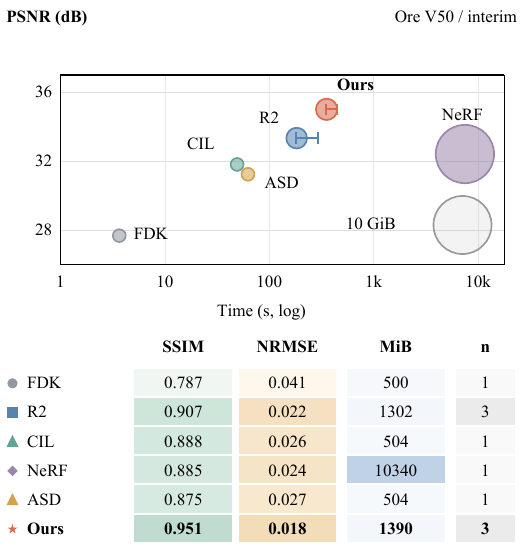}
\caption{Ore quality versus log time at 50 views on RTX 5090 with method-specific budgets. Bubble area encodes maximum sampled process memory across repeats, not CUDA peak. The lower matrix gives SSIM, unscaled NRMSE, MiB, and repeat count. For $n=3$, time is the median and quality is the mean.}
\label{fig:efficiency_ore}
\end{figure}

On the 50-view ore case, Fig.~\ref{fig:efficiency_ore} compares reconstruction quality and observed cost on the same RTX 5090 GPU. Primary time spans worker start to final metric. The proposed method attains 35.01~dB mean PSNR across 3 repeats, versus 33.34~dB for R2+NNLS, while its median primary time is 354.514~s versus 183.081~s. The observed primary-time ranges are 348.035--449.293~s and 180.119--292.759~s, respectively. For these 2 methods, quality is the repeat mean and time the median. The remaining 4 methods have 1 observation each.

\subsection{Ablation Studies}
\label{sec:exp_ablation}

\begin{figure}[!t]
\centering
\includegraphics[width=\columnwidth]{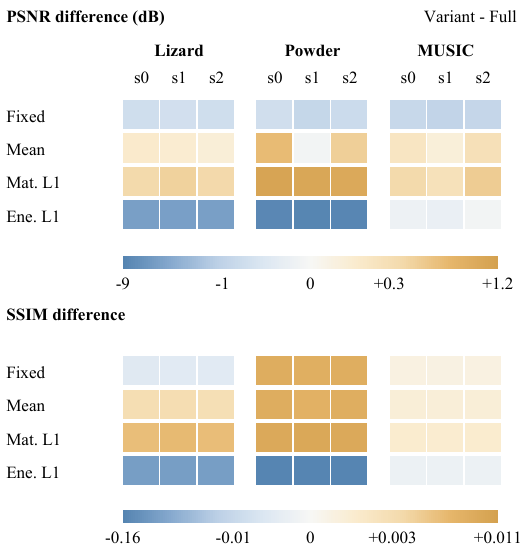}
\caption{Paired variant-minus-Full contrasts across 3 cases and seeds. Upper/lower maps show PSNR (dB)/SSIM differences. Blue is negative, ochre positive. Signed-asinh scales are centered at zero, with parameters 0.3~dB and 0.003 SSIM.}
\label{fig:ablation_size_quality}
\end{figure}

We compare the complete method with 4 controlled variants on lizard head ($E=M=3$), powder phantom ($E=M=4$), and MUSIC 3D ($E=M=6$). The variants hold the Gaussian set fixed, replace maximum with mean gradient aggregation, use a material-domain $\ell_1$ objective, or use an energy-domain $\ell_1$ objective. We use 3 paired seeds per case with the same saved 50,000-point initialization across settings within each case and seed. All settings run for 15,000 iterations. Table~\ref{tab:ablation_results} reports the material-volume metrics. Fig.~\ref{fig:ablation_size_quality} shows each variant's paired PSNR and SSIM differences from the complete setting.

Material-domain $\ell_1$ has the highest 3-seed mean PSNR and SSIM and the lowest mean NRMSE in each case. Mean aggregation also improves all 3 mean metrics relative to the complete setting, although its powder-phantom PSNR is slightly lower for 1 seed. Holding the Gaussian set fixed reduces mean PSNR in all cases, while its mean SSIM exceeds the complete setting for powder and MUSIC 3D. The energy-domain $\ell_1$ variant performs worse on all 3 mean metrics. Final representation sizes differ across settings, so these results do not isolate the quality effect of aggregation from the number of Gaussians.
\section{Conclusion}

We introduced Gaussian material fields for volumetric multi-energy CT decomposition, combining shared anisotropic geometry with independent material coefficients. A differentiable spectral forward model enabled joint fitting from multi-energy projections, while material-aware density control adapted local capacity using material-specific evidence. Experiments using synthesized projections from public-data-derived pseudo-reference material maps showed improved material reconstruction and localized detail recovery, supported by material-wise comparisons and ablations. Runtime and memory measurements demonstrated computational benefits in the tested settings. This representation unified spatial organization, material composition, and local adaptation.

\appendix
\label{app:gmf_derivations}

This appendix provides the derivations and proofs supporting the method in Sec.~\ref{sec:method}, together with parameterization and implementation details.

\begin{proposition}[Finite Gaussian ray response]
\label{prop:gmf_ray}
For a positive definite $\boldsymbol\Sigma_i$, a unit ray direction, and fixed limits $t_r^-<t_r^+$, define $\beta_{i,r}=\mathbf d_r^{\mathsf T}\mathbf Q_i\boldsymbol\delta_{i,r}$ and $\gamma_{i,r}=\boldsymbol\delta_{i,r}^{\mathsf T}\mathbf Q_i\boldsymbol\delta_{i,r}$. The exact finite response is
\begin{equation}
\begin{aligned}
 \ell_{i,r}^{\mathrm{seg}}
 &=\sqrt{\frac{\pi}{2\alpha_{i,r}}}e^{-\Delta_{i,r}/2}
 \left[\operatorname{erf}(z_{i,r}^+/\sqrt2)
 -\operatorname{erf}(z_{i,r}^-/\sqrt2)\right],\\
 z_{i,r}^{\pm}&=\frac{\alpha_{i,r}t_r^{\pm}+\beta_{i,r}}{\sqrt{\alpha_{i,r}}},
 \qquad \Delta_{i,r}=\gamma_{i,r}-\frac{\beta_{i,r}^2}{\alpha_{i,r}}.
\end{aligned}
\label{eq:gmf_segment_integral}
\end{equation}
Its infinite-limit value is Eq.~\eqref{eq:gmf_line_response}.
\end{proposition}
\begin{IEEEproof}
Suppressing subscripts, expand and complete the square:
\begin{equation}
\begin{aligned}
 (\boldsymbol\delta+t\mathbf d)^{\mathsf T}\mathbf Q(\boldsymbol\delta+t\mathbf d)
 &=\alpha t^2+2\beta t+\gamma\\
 &=\alpha(t+\beta/\alpha)^2+\Delta.
\end{aligned}
\label{eq:gmf_complete_square}
\end{equation}
The substitution $z=\sqrt\alpha(t+\beta/\alpha)$ gives Eq.~\eqref{eq:gmf_segment_integral}. The error-function bracket approaches $2$ at infinite limits. If $F$ is the standard normal CDF, the omitted relative tail is
\begin{equation}
\begin{aligned}
 1-\ell^{\mathrm{seg}}/\ell^\infty
 &=F(z^-)+1-F(z^+)\\
 &\leq\operatorname{erfc}(k/\sqrt2),
 \quad z^-\leq-k,\ z^+\geq k.
\end{aligned}
\label{eq:gmf_tail_bound}
\end{equation}
The bound holds for $k>0$ by CDF monotonicity. It separates segment truncation from the later affine projection approximation.
\end{IEEEproof}

\begin{proposition}[Covariance-dependent rectification]
\label{prop:gmf_rectifier}
For the positive definite locally transformed covariance, write
\begin{equation}
 \widetilde{\boldsymbol\Sigma}
 =\begin{bmatrix}\mathbf S&\mathbf b\\\mathbf b^{\mathsf T}&c\end{bmatrix},\qquad
 \xi=c-\mathbf b^{\mathsf T}\mathbf S^{-1}\mathbf b>0.
\label{eq:gmf_schur_covariance}
\end{equation}
Integrating the corresponding unit-peak Gaussian over the third coordinate gives the rectification factor $\sqrt{2\pi\xi}=(2\pi|\widetilde{\boldsymbol\Sigma}|/|\mathbf S|)^{1/2}$ in Eq.~\eqref{eq:gmf_projected_gaussian}.
\end{proposition}
\begin{IEEEproof}
For displacement $[\mathbf h^{\mathsf T},z]^{\mathsf T}$, the block inverse identity yields
\begin{equation}
\begin{aligned}
 \begin{bmatrix}\mathbf h\\z\end{bmatrix}^{\mathsf T}
 \widetilde{\boldsymbol\Sigma}^{-1}
 \begin{bmatrix}\mathbf h\\z\end{bmatrix}
 &=\mathbf h^{\mathsf T}\mathbf S^{-1}\mathbf h
 +\frac{(z-\mathbf b^{\mathsf T}\mathbf S^{-1}\mathbf h)^2}{\xi},\\
 |\widetilde{\boldsymbol\Sigma}|&=|\mathbf S|\xi.
\end{aligned}
\label{eq:gmf_schur_integral}
\end{equation}
The integral of the second Gaussian factor over $z\in\mathbb R$ is $\sqrt{2\pi\xi}$. The first factor is independent of $z$. This proves the stated amplitude and detector-plane exponent. The identity is exact for the affine Gaussian. Its use for perspective rays retains the local approximation in the ray-space transform~\cite{zha2024r2}.
\end{IEEEproof}

\begin{proposition}[Conditional separation and conditioning]
\label{prop:gmf_conditioning}
For fixed geometry and column-major vectorization,
\begin{equation}
\begin{aligned}
 \operatorname{vec}(\widehat{\mathbf P})
 &=(\mathbf B\otimes\boldsymbol\Phi)\operatorname{vec}(\mathbf A),\\
 \operatorname{rank}(\mathbf B\otimes\boldsymbol\Phi)
 &=\operatorname{rank}(\mathbf B)\operatorname{rank}(\boldsymbol\Phi).
\end{aligned}
\label{eq:gmf_identifiability}
\end{equation}
The unrestricted coefficient map is injective precisely when both factors have full column rank, requiring $E\geq M$ and $R\geq N$. For the diagnostic quadratic objective $\mathcal Q(\mathbf A)=\frac12\|\boldsymbol\Phi\mathbf A\mathbf B^{\mathsf T}-\mathbf P\|_F^2$, its coefficient Hessian obeys
\begin{equation}
\begin{aligned}
 \mathbf H_A&=(\mathbf B^{\mathsf T}\mathbf B)\otimes(\boldsymbol\Phi^{\mathsf T}\boldsymbol\Phi),\\
 \kappa_2(\mathbf H_A)&=\kappa_2(\mathbf B)^2\kappa_2(\boldsymbol\Phi)^2
\end{aligned}
\label{eq:gmf_conditioning}
\end{equation}
under those full-column-rank assumptions.
\end{proposition}
\begin{IEEEproof}
The vectorization identity gives the first line of Eq.~\eqref{eq:gmf_identifiability}. Applying thin SVDs to the two factors shows that nonzero singular values of the Kronecker product are their pairwise products. This proves the rank and injectivity claims. The quadratic Hessian is $(\mathbf B\otimes\Phi)^{\mathsf T}(\mathbf B\otimes\Phi)$, giving Eq.~\eqref{eq:gmf_conditioning}. The corresponding unrestricted least-squares perturbation satisfies
\begin{equation}
\begin{aligned}
 \delta\mathbf A_{\mathrm{LS}}
 &=\boldsymbol\Phi^\dagger\delta\mathbf P(\mathbf B^\dagger)^{\mathsf T},\\
 \|\delta\mathbf A_{\mathrm{LS}}\|_F
 &\leq\frac{\|\delta\mathbf P\|_F}
 {\sigma_{\min}(\boldsymbol\Phi)\sigma_{\min}(\mathbf B)}.
\end{aligned}
\label{eq:gmf_coefficient_perturbation}
\end{equation}
The norm bound follows from submultiplicativity.
\end{IEEEproof}

The quadratic model is used only to diagnose conditioning. It is not the Hessian of the actual $\ell_1$ loss or a replacement training objective. Nonnegativity does not generally remove rank-deficient ambiguity at a strictly positive coefficient matrix: a sufficiently small nullspace perturbation preserves positivity. Boundary-constrained cases may differ. These results do not establish global uniqueness with trainable geometry. In particular, identical cloned kernels give duplicate columns of $\Phi$ at insertion. The fixed spectral basis is never jointly learned.

\begin{proposition}[Geometric differentials]
\label{prop:gmf_differentials}
For fixed ray limits, exact Gaussian responses admit the derivatives below. For the rasterized model, Eq.~\eqref{eq:gmf_gradients} uses the full differential of Eq.~\eqref{eq:gmf_projected_gaussian}, including its rectification factor and position-dependent covariance.
\end{proposition}
\begin{IEEEproof}
For $\mathbf y(t)=\mathbf o+t\mathbf d-\boldsymbol\mu$, differentiation of the unit-peak kernel gives
\begin{equation}
 \nabla_{\boldsymbol\mu}K=K\mathbf Q\mathbf y,\qquad
 \nabla_{\boldsymbol\Sigma}K=\tfrac12K\mathbf Q\mathbf y\mathbf y^{\mathsf T}\mathbf Q.
\label{eq:gmf_kernel_derivatives}
\end{equation}
There is no probability-density normalization term. Set $t_0=-\beta/\alpha$, $\bar{\mathbf y}=\boldsymbol\delta+t_0\mathbf d$, $k^{\pm}=K(\mathbf x_r(t^{\pm}))$, and
\begin{equation}
\begin{aligned}
 M_0&=\ell^{\mathrm{seg}},\qquad M_1=(k^--k^+)/\alpha,\\
 M_2&=\{M_0+(t^--t_0)k^--(t^+-t_0)k^+\}/\alpha.
\end{aligned}
\label{eq:gmf_ray_moments}
\end{equation}
Integration by parts identifies $M_j=\int_{t^-}^{t^+}(t-t_0)^jK\,\mathrm dt$ for $j=0,1,2$. Differentiating under the integral in Eq.~\eqref{eq:gmf_kernel_derivatives} therefore yields
\begin{equation}
\begin{aligned}
 \nabla_{\boldsymbol\mu}\ell^{\mathrm{seg}}
 &=\mathbf Q(M_0\bar{\mathbf y}+M_1\mathbf d),\\
 \nabla_{\boldsymbol\Sigma}\ell^{\mathrm{seg}}
 &=\tfrac12\mathbf Q\big[M_0\bar{\mathbf y}\bar{\mathbf y}^{\mathsf T}
   +M_2\mathbf d\mathbf d^{\mathsf T}\\
 &\hspace{3.6em}+M_1(\bar{\mathbf y}\mathbf d^{\mathsf T}
   +\mathbf d\bar{\mathbf y}^{\mathsf T})\big]\mathbf Q.
\end{aligned}
\label{eq:gmf_exact_ray_derivatives}
\end{equation}
The infinite-line case has $M_1=0$ and $M_2=\ell^\infty/\alpha$.

For the rectified response, write $\mathbf S=\widehat{\boldsymbol\Sigma}$, $\mathbf r=\mathbf S^{-1}(\mathbf u-\widehat{\boldsymbol\mu})$, and let $\mathbf E_2=[\mathbf I_2\ \mathbf0]$ select the first two coordinates. Logarithmic differentiation gives
\begin{equation}
\begin{aligned}
 \mathrm d\phi
 &=\phi\big[\mathbf r^{\mathsf T}\mathrm d\widehat{\boldsymbol\mu}
   +\langle\mathbf H,\mathrm d\widetilde{\boldsymbol\Sigma}\rangle_F\big],\\
 \mathbf H
 &=\tfrac12\big[\widetilde{\boldsymbol\Sigma}^{-1}
   +\mathbf E_2^{\mathsf T}(\mathbf r\mathbf r^{\mathsf T}-\mathbf S^{-1})\mathbf E_2\big].
\end{aligned}
\label{eq:gmf_raster_differential}
\end{equation}
The first inverse-covariance term comes from $\log|\widetilde{\boldsymbol\Sigma}|$, and the negative $\mathbf S^{-1}$ term comes from $-\log|\mathbf S|$. Both would be lost by freezing the amplitude correction. With $\mathbf T=\mathbf J\mathbf W_v$ and $\boldsymbol\Lambda=\operatorname{diag}(\boldsymbol\sigma^2)$, the chain rule is
\begin{equation}
\begin{aligned}
 \mathrm d\widetilde{\boldsymbol\Sigma}
 &=\mathbf T\,\mathrm d\boldsymbol\Sigma\,\mathbf T^{\mathsf T}
 +(\mathrm d\mathbf T)\boldsymbol\Sigma\mathbf T^{\mathsf T}
 +\mathbf T\boldsymbol\Sigma(\mathrm d\mathbf T)^{\mathsf T},\\
 \mathrm d\boldsymbol\Sigma
 &=(\mathrm d\mathbf R)\boldsymbol\Lambda\mathbf R^{\mathsf T}
 +\mathbf R\boldsymbol\Lambda(\mathrm d\mathbf R)^{\mathsf T}\\
 &\quad+\mathbf R\operatorname{diag}(2\boldsymbol\sigma\odot\mathrm d\boldsymbol\sigma)\mathbf R^{\mathsf T},\\
 \mathrm d\bar{\mathbf q}
 &=\frac{\mathbf I_4-\bar{\mathbf q}\bar{\mathbf q}^{\mathsf T}}{\|\mathbf q\|_2}\mathrm d\mathbf q,
 \qquad \mathrm d\mathbf T=(\mathrm d\mathbf J)\mathbf W_v.
\end{aligned}
\label{eq:gmf_geometry_differential}
\end{equation}
Here $\mathrm d\mathbf R$ follows the normalized quaternion convention. In perspective geometry $\mathbf J$ depends on the center, so its differential cannot generally be discarded. Substituting these differentials into Eq.~\eqref{eq:gmf_first_variation} proves the parameter gradients in Eq.~\eqref{eq:gmf_gradients}. They hold locally with fixed culling decisions. Hard active-set changes are not differentiated.
\end{IEEEproof}

\begin{proposition}[Material-specific refinement evidence]
\label{prop:gmf_evidence}
For the separate projected-center branches in Eq.~\eqref{eq:gmf_material_evidence}, their sum is the shared projected-center partial derivative. Their maximum norm satisfies
\begin{equation}
 \left\|\frac1M\sum_m\mathbf h_{i,m}\right\|_2
 \leq\frac1M\sum_m\|\mathbf h_{i,m}\|_2
 \leq\max_m\|\mathbf h_{i,m}\|_2.
\label{eq:gmf_max_evidence_bound}
\end{equation}
\end{proposition}
\begin{IEEEproof}
The projected-center derivative of a footprint with fixed covariance is $\nabla_{\widehat{\boldsymbol\mu}}\Phi_{r,i}=\Phi_{r,i}\widehat{\boldsymbol\Sigma}^{-1}\boldsymbol\rho_{i,r}$. Its coefficient in material branch $m$ is $a_{i,m}D_{r,m}$, giving Eq.~\eqref{eq:gmf_material_evidence}. Summing branches gives the shared-center contribution. The triangle inequality and the bound of an arithmetic mean by the largest entry prove Eq.~\eqref{eq:gmf_max_evidence_bound}.
\end{IEEEproof}

The result prevents averaging-induced dilution, but it does not guarantee accurate recovery or optimal allocation. Temporal averaging and material maximization do not commute: responses $(2,0)$ and $(0,2)$ over two iterations have mean-of-maxima $2$ but max-of-means $1$ before denominator damping. Equation~\eqref{eq:gmf_refinement_score} uses the former.

\begin{proposition}[Effect of primitive edits]
\label{prop:gmf_edits}
Under additive rendering with unchanged other primitives, a clone with identical kernels and halved activated coefficients preserves the field and prediction at insertion. The splitting rule in Sec.~\ref{sec:method_densification} conserves coefficient sums but not spatial integrals in general. For any resulting prediction perturbation,
\begin{equation}
 \big|\mathcal L(\widehat{\mathbf P}+\delta\widehat{\mathbf P})
       -\mathcal L(\widehat{\mathbf P})\big|
 \leq\frac1R\|\delta\widehat{\mathbf P}\mathbf W_E\|_{1,1}.
\label{eq:gmf_edit_loss_bound}
\end{equation}
Here $\mathcal L(\widehat{\mathbf P})=\| (\widehat{\mathbf P}-\mathbf P)\mathbf W_E\|_{1,1}/R$.
\end{proposition}
\begin{IEEEproof}
The field perturbation corresponding to Eq.~\eqref{eq:gmf_edit_response} is
\begin{equation}
 \delta\mathbf c_i(\mathbf x)
 =\sum_{j\in\mathcal C_i}K_{i,j}(\mathbf x)\mathbf a_{i,j}-K_i(\mathbf x)\mathbf a_i.
\label{eq:gmf_edit_field}
\end{equation}
For a clone, $K_{i,1}=K_{i,2}=K_i$ and $\mathbf a_{i,1}=\mathbf a_{i,2}=\mathbf a_i/2$, making Eqs.~\eqref{eq:gmf_edit_response} and \eqref{eq:gmf_edit_field} zero. This identity requires identical masks and no amplitude-dependent cutoff.

For a split, the parent is replaced by $N_c=2$ children according to
\begin{equation}
\begin{aligned}
 \boldsymbol\mu_{i,j}
 &=\boldsymbol\mu_i+\mathbf R_i\operatorname{diag}(\boldsymbol\sigma_i)\boldsymbol\epsilon_j,
 \quad \boldsymbol\epsilon_j\sim\mathcal N(\mathbf0,\mathbf I_3),\\
 \boldsymbol\sigma_{i,j}&=\boldsymbol\sigma_i/k_c,\qquad k_c=0.8N_c,\\
 \mathbf R_{i,j}&=\mathbf R_i,\qquad \mathbf a_{i,j}=\mathbf a_i/N_c.
\end{aligned}
\label{eq:gmf_split}
\end{equation}
For the unnormalized kernels in Eq.~\eqref{eq:gmf_parameterization},
\begin{equation}
\begin{aligned}
 \mathbf m_i&:=\int_{\mathbb R^3}K_i(\mathbf x)\mathbf a_i\,\mathrm d\mathbf x
 =(2\pi)^{3/2}|\boldsymbol\Sigma_i|^{1/2}\mathbf a_i,\\
 \sum_j\mathbf a_{i,j}&=\mathbf a_i,\qquad
 \sum_j\mathbf m_{i,j}=k_c^{-3}\mathbf m_i.
\end{aligned}
\label{eq:gmf_split_integral}
\end{equation}
The second identity follows from $\boldsymbol\Sigma_{i,j}=\boldsymbol\Sigma_i/k_c^2$, provided no additional scale clipping changes Eq.~\eqref{eq:gmf_split}. For $N_c=2$, the whole-space integral ratio is $1/1.6^3\approx0.24414$, not one. Finally, applying $||x+y|-|x||\leq|y|$ entrywise proves Eq.~\eqref{eq:gmf_edit_loss_bound}.
\end{IEEEproof}

A clone does not increase the column span at insertion, and identical copies may remain identical if all updates and optimizer states are symmetric. A split generally changes the loss, with no sign guaranteed for that change. These properties are not convergence claims.

For pruning a primitive, $\delta\widehat{\mathbf P}_i=-\boldsymbol\phi_i(\mathbf B\mathbf a_i)^{\mathsf T}$. Nonnegativity implies the exact perturbation size
\begin{equation}
 \frac1R\|\delta\widehat{\mathbf P}_i\mathbf W_E\|_{1,1}
 =\frac{\|\boldsymbol\phi_i\|_1}{R}\mathbf w^{\mathsf T}\mathbf B\mathbf a_i,
\label{eq:gmf_prune_perturbation}
\end{equation}
where $\mathbf w=[w_1,\ldots,w_E]^{\mathsf T}$. If all amplitudes are below $\tau_a$, replace $\mathbf a_i$ by $\tau_a\mathbf1_M$ to obtain an upper bound. Thus small peak amplitude alone is not a scale-independent error guarantee. Similarly, a center outside $\Omega$ need not have zero Gaussian contribution inside it.

\textit{Parameter maps and projection-only initialization.}
With finite logits, the scale parameterizations are
\begin{equation}
 \boldsymbol\sigma_i=
 \begin{cases}
 \sigma_{\min}\mathbf1_3+(\sigma_{\max}-\sigma_{\min})\operatorname{sigmoid}(\boldsymbol\eta_i),
 &\text{bounded},\\
 \exp(\boldsymbol\eta_i),&\text{otherwise},
 \end{cases}
\label{eq:gmf_scale_maps}
\end{equation}
where $0<\sigma_{\min}<\sigma_{\max}$. Their differentials are obtained elementwise: $\mathrm d\sigma=\sigma\,\mathrm d\eta$ in the exponential branch, or $\mathrm d\sigma=(\sigma_{\max}-\sigma_{\min})s(1-s)\mathrm d\eta$ with $s=\operatorname{sigmoid}(\eta)$ in the bounded branch.

Initialization uses only measured projections:
\begin{equation}
\begin{aligned}
 V_e^{(0)}&=\operatorname{FDK}(\mathbf p_e),\\
 \mathbf c^{(0)}(\mathbf x)
 &\in\underset{\mathbf z\geq0}{\operatorname{argmin}}\,
 \|\mathbf B\mathbf z-\mathbf v^{(0)}(\mathbf x)\|_2^2.
\end{aligned}
\label{eq:gmf_initialization_nnls}
\end{equation}
Set $s^{(0)}=\mathbf1_M^{\mathsf T}\mathbf c^{(0)}$ and $\Omega_\tau=\{\mathbf x:s^{(0)}(\mathbf x)>\tau_{\mathrm{init}}\}$. Initial centers are sampled from $\Omega_\tau$. When candidates are insufficient, highest-score locations and sampling with replacement are permitted. With nearest-neighbor distance $d_i=\min_{j\ne i}\|\boldsymbol\mu_i^{(0)}-\boldsymbol\mu_j^{(0)}\|_2$, the initialization targets are
\begin{equation}
\begin{aligned}
 \widetilde{\mathbf a}_i^{(0)}&=\eta_{\mathrm{init}}\mathbf c^{(0)}(\boldsymbol\mu_i^{(0)}),\\
 \widetilde{\boldsymbol\sigma}_i^{(0)}
 &=\operatorname{clip}(d_i,\sigma_{\min},\sigma_{\max})\mathbf1_3,
 \qquad \mathbf R_i^{(0)}=\mathbf I_3.
\end{aligned}
\label{eq:gmf_initialization_targets}
\end{equation}
For positive amplitudes the inverse-Softplus map is
\begin{equation}
 \beta=\log(e^a-1)=a+\log(1-e^{-a}),\qquad a>0.
\label{eq:gmf_inverse_softplus}
\end{equation}
A zero NNLS coefficient is not representable by finite Softplus logits, and a clipped scale at a sigmoid endpoint requires an interior numerical convention. Finite initialization therefore requires a positive numerical floor and an interior convention for bounded scales. The tildes in Eq.~\eqref{eq:gmf_initialization_targets} distinguish targets from finite activated parameters. Cloning and splitting divide activated coefficients, not their logits. Bound-violating split scales likewise require explicit handling. The integral identity in Eq.~\eqref{eq:gmf_split_integral} applies to the stated, unclipped split rule.

\textit{Continuous queries and computational accounting.}
For arbitrary query locations, set $[\mathbf K_Q]_{q,i}=K_i(\mathbf x_q;\Theta_T)$. The final material and energy samples are
\begin{equation}
 \mathbf C_Q=\mathbf K_Q\mathbf A_T,\qquad
 \mathbf V_Q=\mathbf C_Q\mathbf B^{\mathsf T}.
\label{eq:gmf_continuous_query}
\end{equation}
These are point samples. Denser queries do not establish additional physical resolution. The stored trainable scalar counts are $N(10+M)$ for the shared representation and $11MN$ for $M$ independent single-material sets of $N$ primitives each. They are not intrinsic degrees of freedom because quaternion normalization is redundant. Separate material rasterization with $K_v$ effective Gaussian--pixel interactions per view costs
\begin{equation}
 \begin{aligned}
 \text{view evaluation: }&\mathcal O(MK_v)+\mathcal O(EMHW),\\
 \text{trainable storage: }&\mathcal O(N(10+M)).
 \end{aligned}
\label{eq:gmf_computational_cost}
\end{equation}
The second work term is mixing on an $H\times W$ detector. Sorting/culling, optimizer states, initialization, and output arrays are additional costs. Explicitly storing material samples requires $\mathcal O(QM)$ space. No fused multi-channel acceleration or material-count-independent runtime is assumed.

\bibliographystyle{IEEEtran}
\bibliography{reference}

\end{document}